%% file: arxiv.tex
\documentclass[11pt,letterpaper,a4paper]{article}
\usepackage[utf8]{inputenc}
\usepackage{amsmath}
\usepackage{amsfonts}
\usepackage{amssymb}
\usepackage{hyperref}       
\usepackage{url}            
\usepackage{booktabs}       
\usepackage{amsfonts}       

\title{Graph-based Discriminators: Sample Complexity and Expressiveness}

\author{
  Roi Livni \\
  Tel Aviv University\\
  \texttt{rlivni@tauex.tau.ac.il} \\
   \and
   Yishay Mansour \\
   Tel Aviv University \\
   \texttt{mansour.yishay@gmail.com} \\
}
\input{header}
\date{}
\begin{document}
\maketitle

\begin{abstract}
A basic question in learning theory is to identify if two
distributions are identical when we have access only to examples sampled from the distributions.
This basic task is considered, for example, in the context of
Generative Adversarial Networks (GANs), where a discriminator is trained to distinguish between a real-life distribution and a synthetic distribution.
Classically, we use a hypothesis class $H$ and claim that the two
distributions are distinct if for some $h\in H$ the expected value
on the two distributions is (significantly) different.

Our starting point is the following fundamental problem: "is having
the hypothesis dependent on more than a single random example
beneficial". To address this challenge we define $k$-ary based
discriminators, which have a family of Boolean $k$-ary functions
$\G$. Each function $g\in \G$ naturally defines a hyper-graph,
indicating whether a given hyper-edge exists. A function $g\in \G$
distinguishes between two distributions, if the expected value of
$g$, on a $k$-tuple of i.i.d examples, on the two distributions is
(significantly) different.

We study the expressiveness of families of $k$-ary functions,
compared to the classical hypothesis class $H$, which is $k=1$. We
show a separation in expressiveness of $k+1$-ary versus $k$-ary
functions. This demonstrate the great benefit of having $k\geq 2$ as
distinguishers.

For $k\geq 2$ we introduce a notion similar to the VC-dimension, and
show that it controls the sample complexity. We proceed and provide upper and
lower bounds as a function of our extended notion of VC-dimension.

\end{abstract}

\section{Introduction}
The task of discrimination consists of a \emph{discriminator} that
receives finite samples from two distributions, say $p_1$ and $p_2$,
and needs to certify whether the two distributions are distinct.
Discrimination has a central role within the framework of Generative
Adversarial Networks \cite{goodfellow2014generative}, where a
discriminator trains a neural net to distinguish between samples
from a real-life distribution and samples generated synthetically by
another neural network, called a \emph{generator}.

A possible formal setup for discrimination identifies the discriminator with some
distinguishing class $\D=\{f:X\to \mathbb{R}\}$ of
\emph{distinguishing functions}. In turn, the discriminator wishes
to find the best $d\in \D$ that distinguishes between the two
distributions. Formally, she wishes to find $d\in \D$ such that\footnote{
Note that with such $d$ at hand, with an order of $O(1/\epsilon^2)$
examples one can verify if any discriminator in the class certifies
that the two distributions are distinct.}
\begin{equation}\label{eq:distinct}
\left|\EE_{x\sim p_1}[d(x)]-\EE_{x\sim p_2}[d(x)] \right| >\sup_{d^*\in \D}\left|\EE_{x\sim p_1}[d^*(x)]-\EE_{x\sim p_2}[d^*(x)] \right|
-\epsilon.
\end{equation}
For examples, in GANs, the class of distinguishing functions we will consider could be the class of neural networks trained by the discriminator.

The first term in the RHS of \cref{eq:distinct} is often referred to
as the \emph{Integral Probability Metric} (IPM distance) w.r.t a
class $\D$ \cite{muller1997integral}, denoted $\ipm_{\D}$. As such, we can think of the
discriminator as computing the $\ipm_{\D}$ distance.

Whether two, given, distributions can be distinguished by the
discriminator becomes, in the IPM setup, a property of the
distinguishing class. Also, the number of examples needed to be
observed will depend on the class in question. Thus, if we take a
large expressive class of distinguishers, the discriminator can
potentially distinguish between any two distributions that are far
in total variation. In that extreme, though, the class of
distinguishers would need to be very large and in turn, the number
of samples needed to be observed scales accordingly. One could also
choose a ``small" class, but at a cost of smaller distinguishing
power that yields smaller IPM distance.

For example, consider two distributions over $[n]$ to be
distinguished. We could choose as a distinguishing class the class
of \emph{all} possible subsets over $n$. This distinguishing class
give rise to the total variation distance, but the sample complexity
turns out to be $O(n)$. Alternatively we can consider the class of
\emph{singletones}: This class will induce a simple IPM distance,
with graceful sample complexity, however in worst case the IPM
distance can be as small as $O(1/n)$ even though the total variation
distance is large.

Thus, IPM framework initiates a study of generalization complexity
where we wish to understand what is the expressive power of each
class and what is its sample complexity.

For this special case that $\D$ consists of Boolean functions, the
problem turns out to be closely related to the classical statistical
learning setting and prediction \cite{vapnik1968uniform}. The sample
complexity (i.e., number of samples needed to be observed by the
discriminator) is governed by a combinatorial measure termed
\emph{VC dimension}. Specifically, for the discriminator to be able
to find a $d$ as in \cref{eq:distinct}, she needs to observe order
of $\Theta(\frac{\vcd}{\epsilon^2})$ examples, where $\vcd$ is the
VC dimension of the class $\D$ \cite{blumer1989learnability,
vapnik1968uniform}.

In this work we consider a natural extension of this framework to
more sophisticated discriminators: For example, consider a
discriminator that observes pairs of points from the distribution
and checks for collisions -- such a distinguisher cannot apriori be
modeled as a test of Boolean functions, as the tester measures a
relation between two points and not a property of a single point. The collision test has indeed been used, in the context of synthetic data generation, to evaluate the \emph{diversity} of the synthetic distribution \cite{arora2017gans}.

More generally, suppose we have a class of $2$-ary Boolean functions: $\G=\{g: g(x_1,x_2)\to \{0,1\}\}$ and the discriminator wishes to (approximately) compute
\begin{align}\label{eq:graph_distinct} \sup_{g\in \G}\left|\EE_{(x_1,x_2)\sim p^2_1}[g(x_1,x_2)]-\EE_{(x_1,x_2)\sim p^2_2}[g(x_1,x_2)] \right|.\end{align}

Here $p^2$ denotes the product distribution over $p$. More
generally, we may consider $k$-ary mappings, but for the sake of
clarity, we will restrict our attention in this introduction to
$k=2$. 

Such $2$-ary Boolean mapping can be considered as graphs where
$g(x_1,x_2)=1$ symbolizes that there exists an edge between $x_1$
and $x_2$ and similarly $g(x_1,x_2)=0$ denotes that there is no such
edge. The collision test, for example, is modelled by a graph that contains only self--loops. We thus call such multi-ary statistical tests \emph{graph-based
distinguishers}.

Two natural question then arise

\begin{enumerate}
\item\label{q:1} Do graph--based discriminators have any added distinguishing power over classical discriminators?
\item\label{q:2} What is the sample complexity of graph--based discriminators?
\end{enumerate}
With respect to the first question we give an affirmative answer and
we show a separation between the distinguishing power of
graph--based discriminators and classical discriminators. As to the
second question, we introduce a new combinatorial measure (termed
\emph{graph VC dimension}) that governs the sample complexity of
graph--based discriminators -- analogously to the VC characterization
of the sample complexity of classical discriminators. We next elaborate on each of these two results.

As to the distinguishing power of graph--based discriminators, we
give an affirmative answer in the following sense: We show that
there exists a single graph $g$ such that, for any distinguishing
class $\D$ with bounded VC dimension, and $\epsilon$, there are two
distributions $p_1$ and $p_2$ that are $\D$--indistinguishable but
$g$ certifies that $p_1$ and $p_2$ are distinct. Namely, the
quantity in \cref{eq:graph_distinct} is at least $1/4$ for
$\G=\{g\}$.

This result may be surprising. It is indeed known that for any two
distributions that are $\epsilon$--far in total variation, there
exists a boolean mapping $d$ that distinguishes between the two
distributions. In that sense, distinguishing classes are known to be
universal. Thus, asymptotically, with enough samples any two
distribution can be ultimately distinguished via a standard
distinguishing function.

Nevertheless, our result shows that, given finite data,   the
restriction to classes with finite capacity is limiting, and there
could be graph-based distinguishing functions whose distinguishing
power is not comparable to \emph{any} class with finite capacity. We
stress that the same graph competes with \emph{all} finite--capacity
classes, irrespective of their VC dimension.

\ignore{The last result is given over infinite set of vertices: This is of
course necessary if we want to show such a strong separation.
Indeed, over a finite domain of size $n$, we can calculate the total
variation with sample complexity $O(n)$. We thus, also, complement
our results and give a more refined analysis that shows a separation
between graph based discriminators and classes with VC dimension
smaller than $O(\epsilon^2 \mathrm{poly}(\log n))$.}

With respect to the second question, we introduce a new VC-like
notion termed \emph{graph VC dimension} that extends naturally to
graphs (and hypergraphs). On a high level, we show that for a class of graph-based distinguishers with graph VC dimension $\vcd$, $O(\vcd)$ examples are sufficient for discrimination and that $\Omega(\sqrt{\vcd})$ examples are necessary. This leaves a gap of factor $\sqrt{\vcd}$ which we leave as an open question.

The notion we introduce is strictly weaker than the standard
VC--dimension of families of multi-ary functions, and the proofs we
provide do not follow directly from classical results on
learnability of finite VC classes \cite{vapnik1968uniform,
blumer1989learnability}. In more details, a graph-based distinguishing class $\G$ is  a
family of Boolean functions over the product space of vertices $\V$:
$\G \subseteq \{0,1\}^{\V^2}$. As such it is equipped with a VC
dimension, the largest set of pairs of vertices that is shattered by
$\G$.

It is not hard to show that finite VC is sufficient to achieve
finite sample complexity bounds over $2$-ary functions
\cite{clemenccon2008ranking}. It turns out, though, that it is not a
necessary condition: For example, one can show that the class of
\emph{k-regular graphs} has finite graph VC dimension but infinite
VC dimension. Thus, even though they are not learnable in the
standard PAC setting, they have finite sample complexity within the
framework of discrimination.

The reason for this gap, between learnability and discriminability,
is that learning requires uniform convergence with respect to any
possible distribution over pairs, while discrimination requires
uniform convergence only with respect to product distributions --
formally then, it is a weaker task, and, potentially, can be
performed even for classes with infinite VC dimension.

\subsection{Related Work}
The task of discrimination has been considered as early as the work
of Vapnik and Chervonenkis in \cite{vapnik1968uniform}. In fact,
even though Vapnik and Chervonenkis original work is often referred
in the context of prediction, the original work considered the
question of  when the empirical frequency of Boolean functions
converges uniformly to the true probability over a class of
functions. In that sense, this work can be considered as a natural
extension to $k$-ary functions and generalization of the notion of
VC dimension.

The work of \cite{clemenccon2008ranking, clemenccon2016scaling}
studies also a generalization of VC theory to multi-ary functions in the context of ranking tasks and U-statistics. They study the
standard notion of VC dimension. Specifically they  consider the
function class as Boolean functions over multi-tuples and the VC
dimension is defined by the largest set of multi-tuples that can be
shattered. Their work provides several interesting fast-rate
convergence guarantees. As discussed in the introduction, our notion
of capacity is weaker, and in general the results are incomparable.

\paragraph{GANs}
A more recent interest in discrimination tasks is motivated by the
framework of GANs, where a neural network is trained to distinguish
between two sets of data -- one is real and the other is generated
by another neural network called \emph{generator}. Multi-ary tests have been proposed to assess the quality of GANs
networks. \cite{arora2017gans} suggests birthday paradox to evaluate
\emph{diversity} in GANs. \cite{richardson2018gans} uses Binning to assess the solution proposed by GANs.

Closer to this work \cite{lin2018pacgan} suggests the use of a
discriminator that observes samples from the $m$-th product
distribution. Motivated by the problem of \emph{mode collapse} they
suggest a theoretical framework in which they study the algorithmic
benefits of such discriminators and observe that they can
significantly reduce mode collapse. In contrast, our work is less
concerned with the problem of mode collapse directly and we ask in
general if we can boost the distinguishing power of discriminators via multi-ary discrimination. Moreover, we provide
several novel sample complexity bounds.

\paragraph{Property Testing}
A related problem to ours is that of testing closeness of
distributions \cite{batu2000testing, goldreich2011testing}. Traditionally,
testing closeness of distribution is concerned with evaluating if two discrete distributions
are close vs. far/identical in \emph{total variation}.
\cite{goldreich2011testing}, motivated by graph expansion test,
propose a collision test to verify if a certain distribution is
close to uniform. Interestingly, a collision test is a graph-based
discriminator which turns out to be optimal for the
setting\cite{paninski2008coincidence}. Our sample--complexity lower
bounds are derived from these results. Specifically we reduce
discrimination to testing uniformity \cite{paninski2008coincidence}.
Other lower bounds in the literature can be similarly used to
achieve alternative (yet incomparable bounds) (e.g.
\cite{chan2014optimal} provides a $\Omega(n^{2/3}/\epsilon^{3/4})$
lower bounds for testing whether two distributions are far or
close).

In contrast with the aforementioned setup, here we do not measure distance between distributions in terms of total
variation but in terms of an IPM distance induced by a class of
distinguishers. The advantage of the IPM distance is that it
sometimes can be estimated with limited amount of samples, while the
total variation distance scales with the size of the support, which
is often too large to allow estimation.

Several works do study the question of distinguishing between two
distributions w.r.t a finite capacity class of tests, Specifically
the work of \cite{kothari2018agnostic} studies refutation algorithms
that distinguish between noisy labels and labels that correlate with
a bounded hypothesis class. \cite{vadhan2017learning} studies a
closely related question in the context of realizable PAC learning.
A graph-based discriminator can be directly turned to a refutation
algorithm, and both works of \cite{kothari2018agnostic,
vadhan2017learning} show  reductions from refutation to learning. In
turn, the agnostic bounds of \cite{kothari2018agnostic} can be
harnessed to achieve lower bounds for graph-based discrimination.
Unfortunately this approach leads to suboptimal lower bounds. It
would be interesting to see if one can improve the guarantees for
such reductions, and in turn exploit it for our setting.

\section{Problem Setup}
\subsection{Basic Notations -- Graphs and HyperGraphs}
Recall that a $k$-hypergraph $g$ consists of a a set $\V_g$ of
\emph{vertices} and a collection of non empty $k$--tuples over $\V$:
$E_g\subseteq \V^k$, which are referred to as \emph{hyperedges}. If
$k=2$ then $g$ is called a graph. $1$--hypergraphs are simply
identified as subsets over $\V$. We will normally use $d$ to denote
such $1$-hypergraphs and will refer to them as
\emph{distinguishers}. A distinguisher $d$ can be identified with a
Boolean function according to the rule:  $d(x)=1$ iff $x\in E_d$.

Similarly we can identify a $k$-hypergraph with a function $g: \V^k
\to \{0,1\}$. Namely, for any graph $g$ we identify it with the
Boolean function
\[g(v_1,\ldots, v_k)=
\begin{cases}
1 & (v_1,\ldots, v_k)\in E_g
\\
0 & \mathrm{else}
\end{cases}
\]

We will further simplify and assume that $g$ is \emph{undirected}, this means that for any permutation $\pi:[k]\to [k]$, we have that
\[g(v_{\pi(1)}, v_{\pi(2)},\ldots ,v_{\pi(k)})= g(v_1,\ldots, v_k).\]

We will call undirected $k$-hypergraphs, $k$-distinguishers. A collection of $k$-distinguishers over a common set of vertices $\V$ will be referred to as a $k$-\emph{distinguishing class}. If $k=1$ we will simply call such a collection \emph{a distinguishing class}. For $k>1$ we will normally denote such a collection with $\G$ and for $k=1$ we will often use the letter $\D$.

Next, given a distribution $\P$ over vertices and a $k$--hypergraph $g$ let us denote as follows the frequency of an edge w.r.t $\P$:

\[\L_{P}(g)= \EE_{\vv_{1:k}\sim  \P^k}\left[ g(\vv_{1:k})\right]=P^k\left[\{(\vv_1,\ldots,\vv_k): (\vv_1,\ldots, \vv_k)\in E_{g})\}\right],\]
 where we use the notation $\vv_{1:t}$ in shorthand for the sequence $(\vv_1,\ldots, \vv_t)\in \V^t$, and $P^k$ denotes the product distribution of $P$ $k$ times.

Similarly, given a sample $S=\{v_i\}_{i=1}^m$ we denote the empirical frequency of an edge:
\[\L_{S}(g)= \frac{1}{m^k} \sum_{\uu_{1:k}\in S^k} g(\uu_{1:k})= \frac{| \{(\uu_1,\ldots, \uu_k)\in E_g: \forall i,~ \uu_i \in S\}|}{m^k}\]

As a final set of notations: Given a $k$-hypergraph $g$ a sequence $\vv_{1:n}$ where $n<k$, we define a $k-n$--distinguisher $g_{\vv_{1:n}}$ as follows:
\[ g_{\vv_{1:n}}(\uu_{1:k-n}) = g(\vv_1,\ldots \vv_n, \uu_{1},\ldots \uu_{k-n}).\]

In turn, we define the following distinguishing classes: For every sequence $\vv_{1:n}$, $n<k$, the distinguishing class $\G_{\vv_{1:n}}$ is defined as follows:
\begin{align}\label{eq:Gv}
\G_{\vv_{1:n}}= \{g_{\vv_{1:n}}: g\in \G\}
\end{align}

Finally, we point out that we will mainly be concerned with the case that $|\V|\le \infty$ or $\V=\mathbb{N}$. However, all the results here can be easily extended to other domains as long as certain (natural) measurability assumptions are given to ensure that VC theory holds (see \cite{vapnik1968uniform, ben20152}).

\subsection{IPM distance}

Given a class of distinguishers $\D$ the induced IPM distance \cite{muller1997integral}, denoted by $\ipm_{\D}$, is a (pseudo)--metric between distributions over $\V$ defined as follows
\[\ipm_{\D}(p_1,p_2) = \sup_{d\in \D}|\L_{p_1}(d)-\L_{p_2}(d)|= \sup_{d\in \D}\left| \EE_{v\sim p_1}[d(v)]-\EE_{v\sim p_2}[d(v)]\right|.\]

The definition can naturally be extended to a general family of graphs, and we define:

\[ \ipm_{\G}(p_1,p_2)= \sup_{g\in \G}\left| \L_{p_1}(g)-\L_{p_2}(g)]\right|=\sup_{g\in \G}\left| \EE_{\vv_{1:k}\sim p_1^k} [g(\vv_{1:k})]-\EE_{\vv_{1:k}\sim p_2^k} [g(\vv_{1:k})]]\right|\]
Another metric we would care about is the \emph{total variation metric}. Given two distributions $p_1$ and $p_2$ the total variation distance is defined as:
\[ \tv{p_1}{p_2}= \sup_{E} |p_1(E)-p_2(E)|\] where $E\subseteq \V^{\{0,1\}}$ goes over all measurable events.

In contrast with an IPM distance, the total variation metric is
indeed a metric and any two distributions $p_1 \ne p_2$ we have that
$\tv{p_1}{p_2}>0$. In fact, for every distinguishing class $\D$,
$\ipm_{\D} \preceq \mathrm{TV}$.\footnote{we use the notation $f_1 \preceq
f_2$ to denote that for every $x,y$ we have $f_1(x,y)\le f_2(x,y)$.}

For finite classes of vertices $\V$, it is known that the total variation metric is given by
\[\tv{p_1}{p_2}=\frac{1}{2}\sum_{v\in \V}|p_1(v)-p_2(v)|.\] Further, if we let $\D= P(\V)$ the power set of $\V$ we obtain
\[\ipm_{P(\V)}(p_1,p_2)=\tv{p_1}{p_2}.\]

\subsection{Discriminating Algorithms}\label{sec:disciminating}
\begin{definition}
Given a distinguishing class $\G$ a $\G$-discriminating algorithm
$A$ with sample complexity $m(\epsilon,\delta)$ is an algorithm that
receives as input two finite samples $S=(S_1,S_2)$ of vertices and
outputs a hyper-graph $g_S^A\in \G$ such that:

If $S_1, S_2$ are drawn IID from some unknown distributions
$p_1,p_2$ respectively and $|S_1|,|S_2|> m(\epsilon,\delta)$ then
w.p. $(1-\delta)$ the algorithm's output satisfies:

\[|\L_{p_1}(g_S^A)-\L_{p_2}(g_S^A)|> \ipm_{\G}(p_1,p_2)-\epsilon.\]

The sample complexity of a class $\G$ is then given by the smallest
possible sample complexity of a $\G$-discriminating algorithm $A$.

A class $\G$ is said to be discriminable if it has finite
sample complexity. Namely there exists a discriminating algorithm
for $\G$ with sample complexity $m(\epsilon,\delta)<\infty$.
\end{definition}

\paragraph{VC classes are discriminable}
For the case $k=1$, discrimination is closely related to PAC learning. It is easy to see that a proper learning algorithm for a class $\D$ can be turned into a discriminator: Indeed, given access to samples from two distributions $p_1$ and $p_2$ we can provide a learner with labelled examples from a distribution $p$ defined as follows: $p(y=1)=p(y=-1)=\frac{1}{2}$ and $p(\cdot|y=1)=p_1$, and $p(\cdot|y=-1)=p_2$.
Given access to samples from $p_1$ and $p_2$ we can clearly generate IID samples from the distribution $p$. If, in turn, we provide a learner with samples from $p$ and it outputs a hypothesis $d\in \D$ we have that (w.h.p):
\begin{align*} |\L_{p_1}(d) - \L_{p_2}(d) |&= 2| \frac{1}{2}\EE_{(x,y)\sim p_1\times \{1\}} [yd(x)] + \frac{1}{2}\EE_{(x,y)\sim p_2\times \{-1\}} [yd(x)]| \\
&= 2|\EE_{(x,y)\sim p} [yd(x)]| \\
&= 2(1-2p(d(x)\ne y))\\
& \ge  2(1-2(\min_{d\in \D} p(d(x)\ne y)+\epsilon))\\
& = \max_{d\in \D}(2(|\EE_{(x,y)\sim p} yd(x)|-4\epsilon)\\
& = \max_{d\in \D} |\L_{p_1}(d) - \L_{p_2}(d) | - 4\epsilon\\
& = \ipm_{\D}(p_1,p_2) -4\epsilon
\end{align*}
One can also see that a converse relation holds, if we restrict
our attention to learning balanced labels (i.e.,
$p(y=1)=p(y=-1)$). Namely, given labelled examples from some balanced
distribution, the output of a discriminator is a predictor that
competes with the class of predictors induced by $\D$.

Overall, the above calculation, together with Vapnik and
Chervonenkis's classical result \cite{vapnik1968uniform} shows that
classes with finite VC dimension $\vcd$ are discriminable with
sample complexity $O(\frac{\vcd}{\epsilon^2})$.\footnote{Recall that
the VC dimension of a class $\D$ is the largest set that can be
shattered by $\D$ where a set $S\subseteq \V$ is said to be
shattered if $\D$ restricted to $S$ consists of $2^{|S|}$ possible
Boolean functions.} The necessity of finite VC dimension for
agnostic PAC-learning was shown in \cite{anthony2009neural}.
Basically the same argument shows that given a class $\D$,
$\tilde\Omega(\frac{\vcd}{\epsilon^2})$ examples are necessary for
discrimination. We next introduce a
natural extension of VC dimension to hypergraphs, which will play a
similar role.
\subsection{VC Dimension of hypergraphs}

We next define the notion of graph VC dimension for hypergraphs, as
we will later see this notion indeed characterizes the sample
complexity of discriminating classes, and in that sense it is a
natural extension of the notion of VC dimension for hypotheses
classes:

\begin{definition}
Given a family of $k$-hypergraphs, $\G$: The graph VC dimension of
the class $\G$, denoted $\gvcdim(\G)$, is defined inductively as
follows: For $k=1$ $\gvcdim(\G)$ is the standard notion of VC
dimension, i.e., $\gvcdim(\G)=\vcdim(\G)$. For $k>1$:
\[\gvcdim(\G) = \max_{v\in \V}\{\gvcdim(\G_{v})\}\]
\end{definition}

Roughly, the graph VC dimension of a hypergraph is given by the VC
dimension of the induced classes of distinguishers via projections.
Namely, we can think of the VC dimension of hypergraphs as the
projected VC dimension when we fix all coordinates in an edge except
for one.

\section{Main Results}
We next describe the main results of this work. The results are
divided into two sections: For the first part we characterize the
sample complexity of graph--based distinguishing class. The second
part is concerned with the expressive/distinguishing power of
graph--based discriminators. All proofs are provided in
\fullversion{\cref{prfs:sample,prfs:expressivity} respectively}.

\subsection{The sample complexity of graph-based distinguishing class}

We begin by providing upper bounds to the sample complexity for discrimination

\begin{restatable}[Sample Complexity -- Upper Bound]{theorem}{upper}\label{thm:sampleupper}
Let $\G$ be a $k$--distinguishing class with $\gvcdim(\G)=\vcd$ then
$\G$ has sample complexity $O(\frac{\vcd k^2}{\epsilon^2}\log
1/\delta)$.
\end{restatable}

\Cref{thm:sampleupper} is a corollary of the following uniform convergence upper bound for graph-based distinguishing classes.
\begin{restatable}[uniform convergence]{theorem}{uc}\label{thm:uc}
Let $\G$ be a $k$--distinguishing class with $\gvcdim(\G)=\vcd$. Let
$S=\{v_i\}_{i=1}^m$ be an IID sample of vertices drawn from some
unknown distribution $\P$. If $m=\Omega(\frac{\vcd
k^2}{\epsilon^2}\log 1/\delta)$ then with probability at least
$(1-\delta)$ (over the randomness of $S$):
\[\sup_{g\in \G}\left|\L_{S}(g)-\L_{\P}(g)\right|\le \epsilon.\]
\end{restatable}
\procversion{The proof of \cref{thm:uc} is given in \fullversion{\cref{prf:uc}}.}
We next provide a lower bound for the sample complexity of discriminating algorithms in terms of the graph VC dimension of the class
\begin{restatable}[Sample Complexity -- Lower Bound]{theorem}{low}\label{thm:samplelower}
Let $\G$ be a $k$--distinguishing class with $\gvcdim(\G)=\vcd$. For sufficiently large $\rho$ ($\rho = \Omega(2^{O(k^3)})$), any
$\G$-discriminating algorithm with accuracy $\epsilon>0$ that
succeeds with probability $1-\frac{2^{-k\log k}}{3}$, must observe
at least $\Omega\left(\frac{\sqrt{\vcd}}{2^{7k^3}\epsilon^2}\right)$
samples.
\end{restatable}
\procversion{We refer the reader to \fullversion{\cref{prf:samplelower}} for a proof of \cref{thm:samplelower}.} Our upper bounds and lower bounds leave a gap of order
$O(\sqrt{\vcd})$. As dicussed in \cref{sec:disciminating}, for the case $k=1$ we can provide a tight $\theta(\frac{\vcd}{\epsilon^2})$ bound through a reduction to agnostic PAC learning and the appropriate lower bounds\cite{anthony2009neural}. In general it would be interesting to improve the above bound both in terms of $\rho$ and $k$.

\subsection{The expressive power of graph-based distinguishing class}
So far we have characterized the discriminability of graph-based
distinguishing classes. It is natural though to ask if graph--based
distinguishing classes add any advantage over standard
$1$-distinguishing classes. In this section we provide several
results that show that indeed graph provide extra expressive power
over standard distinguishing classes.

We begin by providing a result over infinite graphs. \procversion{(proof is provided in \fullversion{\cref{prf:main_expressive1}})}
\begin{restatable}{theorem}{expressive}\label{thm:main_expressive1}
Let $\V=\mathbb{N}$. There exists a distinguishing graph class $\G$, with sample complexity $m(\epsilon,\delta)=O(\frac{\log 1/\delta}{\epsilon^2})$ (in fact $|\G|=1$) such that: for any $1$-distinguishing class $\D$ with finite VC dimension, and every $\epsilon>0$ there are two distributions $p_1,p_2$ such that $\ipm_{\D}(p_1,p_2)<\epsilon$ but $\ipm_{\G}(p_1,p_2)>1/2$
\end{restatable}

\Cref{thm:main_expressive1} can be generalized to higher order distinguishing classes \procversion{(see \fullversion{\cref{prf:main_expressivek}} for a proof)}:
\begin{restatable}{theorem}{expressivek}\label{thm:main_expressivek} Let $\V=\mathbb{N}$. There exists a $k$-distinguishing class $\G_k$, with sample complexity $m(\epsilon,\delta)=O(\frac{k^2+\log 1/\delta}{\epsilon^2})$ such that: For any $k-1$-distinguishing class $\G_{k-1}$ with bounded sample complexity, and every $\epsilon>0$ there are two distributions $p_1,p_2$ such that $\ipm_{\G_{k-1}}(p_1,p_2)<\epsilon$ and $\ipm_{\G_{k}}(p_1,p_2)>1/4$.
\end{restatable}
\paragraph{Finite Graphs}
We next study the expressive power of distinguishing graphs over finite domains.

It is known that, over a finite domain $\V=\{1,\ldots, n\}$, we can learn with a sample complexity of $O(\frac{n}{\epsilon^2}\log 1/\delta)$ any distinguishing class. In fact, we can learn the total variation metric (indeed the sample complexity of $\mathcal{P}(\V)$ is bounded by $\log |\mathcal{P}(V)| =n$).

Therefore if we allow classes whose sample complexity scales linearly with $n$ we cannot hope to show any advantage for distinguishing graphs. However, in most natural problems $n$ is considered to be very large (for example, over the Boolean cube $n$ is exponential in the dimension). We thus, in general, would like to study classes that have better complexity in terms of $n$. In that sense, we can show that indeed distinguishing graphs yield extra expressive power.

In particular, we show that for classes with sublogarithmic sample complexity, we can construct graphs that are incomparable with a higher order distinguishing class.

\begin{restatable}{theorem}{expressivefinitek}\label{thm:main_expressive_finite_k}
Let $|\V|=n$. There exists a $k$-distinguishing class $\G_{k}$, with sample complexity $m(\epsilon,\delta)=O(\frac{k^2+\log 1/\delta}{\epsilon^2})$ (in fact $|\G|=1$) such that: For any $\epsilon>0$ and any $k-1$ distinguishing class $\G_{k-1}$ if:
\[ \ipm_{\G_{k-1}} \succ \epsilon \cdot \ipm{\G_{k}}\] then
$\gvcdim(\G_{k-1}) = \Omega( \frac{\epsilon^2}{k^2} \sqrt{\log n})$.
\end{restatable}
\procversion{The proof is given in \fullversion{\cref{prf:main_expressive_finite_k}}.}
We can improve the bound in \cref{thm:main_expressive_finite_k} for the case $k=1$ \procversion{(see \fullversion{\cref{prf:main_expressive_finite_1}} for proof)}.

\begin{restatable}{theorem}{expressivefinite}\label{thm:main_expressive_finite_1}
Let $|\V|=n$. There exists a $2$-distinguishing class $\G$, with sample complexity $m(\epsilon,\delta)=O(\frac{\log 1/\delta}{\epsilon^2})$ (in fact $|\G|=1$) such that: For any $\epsilon>0$ and any distinguishing class $\D$ if:
\[ \ipm_{\D} \succ \epsilon \cdot \ipm_{\G}\] then
$\gvcdim(\D) = \tilde\Omega( \epsilon^2 \log n)$.
\end{restatable}
\section{Discussion and open problems}

In this work we developed a generalization of the standard framework
of discrimination to graph-based distinguishers that discriminate
between two distributions by considering multi-ary tests. Several open question arise from our results:

\paragraph{Improving Sample Complexity Bounds}
In terms of sample complexity, while we give a natural upper bound
of $O(\vcd k^2 )$, the lower bound we provide are  not tight neither
in $d$ nor in $k$ and we provide a lower bound of
$\Omega(\frac{\sqrt{\vcd}}{2^{poly(k)}})$ This leave room for
improvement both in terms of $\vcd$ and in terms of $k$.

\paragraph{Improving Expressiveness Bounds}
We also showed that, over finite domains, we can construct a graph
that is incomparable with any class with VC dimension
$\Omega(\epsilon^2 \log n)$. The best upper bound we can provide
(the VC of a class that competes with any graph) is the naive $O(n)$
which is the VC dimension of the total variation metric.

Additionally, for the $k$-hypergraph case, our bounds deteriorate to
a $\Omega(\epsilon^2\sqrt{\log n})$. The improvement in the graph case
follows from using an argument in the spirit of Boosting
\cite{freund1996game} and Hardcore Lemma \cite{impagliazzo1995hard}
to construct two indistinguishable probabilities with distinct
support over a small domain. It would be interesting to extend these
techniques in order to achieve similar bounds for the $k>2$ case.

\paragraph{Relation to GANs and Extension to Online Setting}
Finally, a central motivation for learning the sample complexity of
discriminators is in the context of GANs. It then raises interesting
questions as to the \emph{foolability} of graph-based distinguishers.

The work of \cite{bousquet2019passing} suggests a framework for
studying sequential games between generators and discriminators
(\emph{GAM-Fooling}). In a nutshell, the GAM setting considers a
sequential game between a generator $G$ that outputs distributions
and a discriminator $D$ that has access to data from some
distribution $p^*$ (not known to $G$). At each round of the game,
the generator proposes a distribution and the discriminator outputs
a $d\in \D$ which distinguishes between the distribution of $G$ and
the true distribution $p^*$. The class $\D$ is said to be
GAM-Foolable if the generator outputs after finitely many rounds a
distribution $p$ that is $\D$--indistinguishable from $p^*$

\cite{bousquet2019passing} showed that a class $\D$ is GAM--foolable
if and only if it has finite Littlestone dimension. We  then ask,
similarly, which classes of graph--based distinguishers are
GAM-Foolable? A characterization of such classes can potentially
lead to a natural extension of the Littlestone notion and online
prediction, to graph-based classes analogously to this work w.r.t VC
dimension

\paragraph{Acknowledgements} The authors would like to thank Shay Moran for helpful discussions and suggesting simplifications for the proofs of \cref{thm:main_expressive1,thm:main_expressivek,thm:main_expressive_finite_k}.
\bibliographystyle{plain}
\bibliography{bibgraph}
\input{appendix-ym}

\end{document}

%% file: header.tex
\usepackage{thmtools}
\usepackage{thm-restate}
\usepackage{amsthm,amsmath,cleveref, cases}
\usepackage{xcolor}
\newcommand{\ignore}[1]{}
\renewcommand{\L}{\mathbb{E}}
\newcommand{\V}{\mathcal{V}}
\newcommand{\G}{\mathcal{G}}
\renewcommand{\H}{\mathcal{H}}
\newcommand{\D}{\mathcal{D}}
\newcommand{\X}{\mathcal{X}}
\renewcommand{\P}{P}
\newcommand{\ipm}{{\mathrm{IPM}}}
\newcommand{\yes}{EQUIVALENT~}
\newcommand{\no}{DISTINCT~}
\newcommand{\EE}{\mathop\mathbb{E}}
\newcommand{\vv}{\mathbf{v}}
\newcommand{\uu}{\mathbf{u}}

\renewcommand{\aa}{\mathbf{a}}
\newtheorem{definition}{Definition}
\newtheorem{lemma}{Lemma}
\newtheorem{theorem}{Theorem}
\newtheorem{corollary}{Corollary}
\newcommand{\tv}[2]{\mathrm{TV}(#1,#2)}

\newcommand{\vcdim}{\mathrm{VC}}
\newcommand{\gvcdim}{\mathrm{gVC}}
\newcommand{\vcd}{\rho}
\newcommand{\roi}[1]{\textcolor{red}{Roi--#1}}

\newtheorem{claim}{Claim}

\newcommand{\fullversion}[1]{#1}
\newcommand{\procversion}[1]{#1}

%% file: appendix-ym.tex
\appendix

\section{Prelimineries and Technical Background}
\subsection{Statistical Learning Theory}
We begin with a brief overview of some classical results in
Statistical Learning theory which characterizes VC classes.
Throughout we assume a domain $\X$ and a \emph{hypothesis class}
which is a family of Boolean functions over $\X$: $\H\subseteq
\{0,1\}^\X$.

\begin{theorem}\label{thm:expectedvapnik}[Within proof of Thm. 6.11 in \cite{uml}]
Let $\H$ be a class with VC dimension $\vcd$ then
\[ \EE_{S\sim P^m} \left[\sup_{h\in \H} |\L_S(h)-\L_{P}(h)|\right] \le \frac{4+\sqrt{\vcd \log(2em/\vcd)}}{\sqrt{2m}}\]
\end{theorem}
Recall that a class $\H$ has the \emph{uniform convergence
property}, if for some $m:(0,1)^2 \to \mathbb{N}$ if $P$ is some
unknown distribution and $S=\{x_i\}_{i=1}^m$ is a sample drawn IID
from $P$ such that $|S|>m(\epsilon,\delta)$ then w.p. $(1-\delta)$
(over the sample $S$):
\[ |\frac{1}{m}\sum_{i=1}^m h(x_i)- \EE_{x\sim P}[h(x)]|<\epsilon\]

The following, high probability analogue of \cref{thm:expectedvapnik}, is also an immediate corollary of Theorem 6.8 in \cite{uml}\footnote{Note that Theorem 6.8 is stated for $0-1$ loss, however considering a distribution with constant label $y=0$ we can reduce the result for the loss $\ell(h,x)=h(x)$}:
\begin{corollary}\label{thm:vapnik}[Within Thm 6.8 \cite{uml}]
Let $\D$ be a class with VC dimension $\vcd$. There exists a
constant $C>0$, such that:\\
Let $p$ be a distribution with finite support over $\V$. Let $S$ be
an IID sequence of $m$ elements drawn from $p$, and denote by $p_S$
the empirical distribution over $S$. If $m\ge C \frac{\vcd+\log
1/\delta}{\epsilon^2}$ then w.p. $(1-\delta)$ (over the random
choice of $S$) we have that
\[\ipm_{\D}(p,p_S)=\sup_{d\in \D} \left|\L_p(d)- \L_{p_S}(d)]\right| <\epsilon\]
\end{corollary}

\subsection{Closeness Testing for Discrete Distribution}
The problem of testing the closeness of two discrete distributions
can be phrased as follows: Given samples from two distributions
$p_1$ and $p_2$ the tester needs to distinguish between the case
$p_1=p_2$ and the case that $\|p_1-p_2\|_1\ge \epsilon$. We will
rely on the following result which follows immediately from a uniformity test lower bound due to \cite{paninski2008coincidence}.

\begin{theorem}\label{thm:distest}
Given $\epsilon>0$ and access to samples from distributions $p_1$
and $p_2$ over $[n]$ any algorithm that returns with probability
$2/3$ $\yes$ if $p_1=p_2$ and returns $\no$ if
$\|p_1-p_2\|_1>\epsilon$ must observe at least $\Omega
\left(\sqrt{n}/\epsilon^2\}\right)$ samples.
\end{theorem}
We note that \cite{chan2014optimal} gives a slightly better lower bound, of an order of $\Omega\left(\max(n^{3/4}/\epsilon^{4/3},\sqrt{n}/\epsilon^2)\right)$. However, to simplify we will focus on rates of $O(1/\epsilon^2)$ that scale quadratically in $\epsilon$.

\section{Sample Complexity --Proofs}\label{prfs:sample}
\subsection{Proof of \cref{thm:uc}}\label{prf:uc}
\uc*

Fix a $k$--distinguishing class $\G$ with graph VC dimension $\vcd$.
As in the standard proof of uniform convergence for VC classes, we
first prove the statement in expectation and then apply Mcdiarmid's
inequality to prove the result w.h.p. Specifically, we will use the
following Lemma (whose proof is given in \cref{prf:expecteduc}):

\begin{lemma}[Uniform Convergence in Expectation]\label{lem:expecteduc}
Let $\G$ be a $k$--distinguishing class with $\gvcdim(\G)=\vcd$. Let
$S=\{v_i\}_{i=1}^m$ be an IID sample of vertices drawn from some
unknown distribution $P$. Then,
\[\EE_{S\sim P^m}\left[ \sup_{g\in \G} |\L_S(g)- \L_P(g)|\right]\le \frac{k\sqrt{4+\vcd\log (2em/\vcd)}}{\sqrt{2m}} + \frac{k(k-1)}{m}
\]\end{lemma}
We next proceed with the proof of \cref{thm:uc}, assuming the correctness of \cref{lem:expecteduc}.
Define
\[ F(S) = \sup_{g\in \G} |\L_{S}(g)-\L_P(g)|,\]
 Let $S=(v_1,\ldots, v_m)$ be a sample and $S'$, some sequence that differ from $S$ only in the $i$-th vertex then we will show that:

\begin{equation}\label{eq:sensitivity}| F(S)-F(S')|\le \frac{2k}{m}\end{equation}
Once we show \cref{eq:sensitivity} holds, the result indeed follow
from Mcdiarmid's inequality and \cref{lem:expecteduc}. Specifically
if we assume that $m\ge
\frac{8k^2(4+\vcd\log(2em/\vcd)}{\epsilon^2}+ \frac{2k^2
1/\delta}{\epsilon^2}$ then we obtain from \cref{lem:expecteduc}
that in expectation:
\[\EE_{S\sim P^m} \sup_{g\in \G} |\L_{S}(g)-\L_P(g)|\le \frac{\epsilon}{2} \]
Applying Mcdiarmid's we obtain that with probability at least $(1-e^{-\frac{m\epsilon^2}{8 k^2}})$, over the sample $S$:
\[ F(S)-\EE_S[F(S)]=\sup_{g\in \G} |\L_{S}(g)-\L_P(g)|-\EE_{S\sim P^m} \sup_{g\in \G} |\L_{S}(g)-\L_P(g)| \le \frac{\epsilon}{2}.\]
Noting that $m> \frac{8k^2 \log 1/\delta}{\epsilon^2}$, we obtain  that with probability at least $(1-\delta)$
\begin{align*}
F(S)= \sup_{g \in \G} |\L_{S}(g)-\L_P(g)| &\le \EE_{S\sim P^m}
\sup_{g\in \G} |\L_{S}(g)-\L_P(g)|+\frac{\epsilon}{2} \le  \epsilon
  \end{align*}
We are thus left with proving that \cref{eq:sensitivity} holds.

For an index $i$ and $m\ge i$, let us denote by $\pi_{i,m}$ all
$k$-subsets of indices from $\{1,\ldots, m\}$ that include $i$ and
we let $\pi_{\neg i,m}$ be all $k$-sequences that do not include
$i$. Given a set $S$ of size $m$ let $S_{i,+}$ all the $k$-subsets
of $S$ that include $v_i$ and let $S_{i,-}$ be all the $k$-subsets
that do not include $v_i$. Next, denote

\[L_{S_{i,+}}(g)= \frac{1}{m^k}\sum_{(i_1,\ldots, i_k)\in \pi_{i,m}}g(\uu_{i_1},\ldots, \uu_{i_k})\]
And similarly
\[L_{S_{i,-}}(g)= \frac{1}{m^k}\sum_{(i_1,\ldots, i_k)\in \pi_{\neg i,m}}g(\uu_{i_1},\ldots, \uu_{i_k})\]
 Then, let $S$ and $S'$ be two samples that differ on the $i$-th example. Specifically assume that $v_i\in S$ and $v'_i \in S'$. Note that $S_{i,-}=S'_{i,-}$.  Then:

\begin{align*}
F(S)-F(S')&=\sup_{g\in \G}|\L_{S}(g)-\L_{P}(g)|-\sup_{g\in \G}|\L_{S'}(g)-\L_{P}(g)|\\
&=\sup_{g\in \G}|L_{S_{i,+}}(g)+L_{S_{i,-}}(g)-\L_{P}(g)|-\sup_{g\in \G}|L_{S'_{i,+}}(g)+L_{S'_{i,-}}(g)-\L_{P}(g)|\\
&\leq \sup_{g\in \G}\left(|L_{S_{i,+}}(g)+L_{S_{i,-}}(g)-\L_{P}(g)|-|L_{S'_{i,+}}(g)+L_{S'_{i,-}}(g)-\L_{P}(g)|\right)\\
&\le \sup_{g\in \G}|\left(L_{S_{i,+}}(g)+L_{S_{i,-}}(g)-\L_{P}(g)\right)- \left(L_{S'_{i,+}}(g)+L_{S'_{i,-}}(g)- \L_{P}(g) \right) |\\
&=\sup_{g\in \G}|L_{S_{i,+}}(g)-L_{S'_{i,+}}(g)|\\
&\le \frac{|S_{i,+}|}{m^k}+\frac{|S'_{i,+}|}{m^k}\\
& = 2\frac{m^k-(m-1)^k}{m^k}\\
&= 2- 2(1-\frac{1}{m})^k  \\
& \le 2\frac{k}{m} \quad (\mathrm{Bernouli's~inequality})
\end{align*}
We are thus left with proving \cref{lem:expecteduc}:

\subsubsection{Proof of \cref{lem:expecteduc}}\label{prf:expecteduc} The proof of the statement follows by induction. The case $k=1$ is the standard uniform convergence property of VC classes, and it follows from \cref{thm:expectedvapnik}.

We next proceed to prove the statement for $k$, assuming it holds for $k-1$. We begin with the following, triangular, inequality:
\begin{align*} &\EE_{S\sim D^{m}}\left[\sup_{g\in \G}|\L_{S}(g)-\L_{P}(g)|\right]\\
= &\EE_{S\sim D^{m}}\left[\sup_{g\in \G}|\L_{S}(g)-\frac{1}{m^{k-1}} \sum_{\vv_{1:k-1}\in S^{k-1}} \EE_{v} g_{\vv_{1:k-1}}(v)+\frac{1}{m^{k-1}} \sum_{\vv_{1:k-1}\in S^{k-1}} \EE_{v} g_{\vv_{1:k-1}}(v)-\L_{P}(g)|\right]\\
\le &\underbrace{\EE_{S\sim D^{m}}\left[\sup_{g\in \G}|\L_{S}(g)-\frac{1}{m^{k-1}} \sum_{\vv_{1:k-1}\in S^{k-1}} \EE_{v} g_{\vv_{1:k-1}}(v)|\right]}_{*}\\
+\\
&\underbrace{\EE_{S\sim D^{m}}\left[\sup_{g\in \G}|\frac{1}{m^{k-1}} \sum_{\vv_{1:k-1}\in S^{k-1}} \EE_{v} g_{\vv_{1:k-1}}(v)-\L_{P}(g)|\right]}_{**}\end{align*}
We next bound the two terms

\paragraph{Bounding *}
\begin{align*}
&\EE_{S\sim P^m} \left[\sup_{g\in \G}\left|\frac{1}{m^{k-1}} \sum_{\vv_{1:k-1}\in S^{k-1}} \frac{1}{m}\sum_{v\in S} g_{\vv_{1:k-1}}(v)-\frac{1}{m^{k-1}} \sum_{\vv_{1:k-1}\in S^{k-1}} \EE_{v} g_{\vv_{1:k-1}}(v))\right|\right] \\
\le & \EE_{S\sim P^m}\left[\frac{1}{m^{k-1}} \sum_{\vv_{1:k-1}\in S^{k-1}} \sup_{d\in \G_{\vv_{1:k-1}}}\left| \frac{1}{m}\sum_{v\in S} d(v)-\EE_{v} d(v)\right|\right]\\
=& \EE_{S\sim P^m}\left[\EE_{\vv_{1:k-1}\sim \mathcal{U}_{S^{k-1}}} \left[\sup_{d\in \G_{\vv_{1:k-1}}}\left| \frac{1}{m}\sum_{v\in S} d(v)-\EE_{v} d(v)\right|\right]\right]
\end{align*}
where we denoted by $\mathcal{U}_{S^{k-1}}$ the uniform distribution over $k-1$-tuples from $S$.
The expectation in the last expression is thus taken w.r.t a process where we pick $m$ elements according to $P$ and then partition them to $m-k+1$ elements and to a sequence $\vv_{1:k-1}$ of distinct elements. This process is equivalent to simply choosing $m-k+1$ elements according to $P$, and then picking $k-1$ new elements, again, according to $P$ as follows:

\begin{align*}
= & \EE_{S\sim P^{m-k+1}}\EE_{(\vv_1,\ldots, \vv_{k-1})\sim P^{k-1}}\left[\sup_{d\in \G_{\vv_{1:k-1}}}| \frac{1}{m}\sum_{v\in S} d(v)+\frac{1}{m}\sum^{k-1}_{i=1}d(\vv_i)-\EE_{v} d(v)|\right]\\
= &\EE_{S\sim P^{m-k+1}}\EE_{(\vv_1,\ldots, \vv_{k-1})\sim P^{k-1}}\left[\sup_{d\in \G_{\vv_{1:k-1}}}| \frac{1}{m}\sum_{v\in S} d(v)  -\EE_{v} d(v)+\frac{1}{m}\sum_{i=1}^{k-1} d(\vv_i)|\right]
\end{align*}
Note that the quantity $\frac{1}{m} \sum d(\vv_i)$ is dependent on $\G_{\vv_{1:k-1}}$, namely these are random sampled choices that depend on our choice of distinguishing class. To bound their effect we next add and subtract auxiliary random variables $\uu_1,\ldots, \uu_{k-1}$ sampled IID according to $P$:
\begin{align*}
=& \EE_{S\sim P^{m-k+1}}\EE_{(\vv_1,\ldots, \vv_{k-1})\sim P^{k-1}}\left[\sup_{d\in \G_{\vv_{1:k-1}}}\left| \frac{1}{m}\sum_{v\in S} d(v) +\frac{1}{m} \EE_{(\uu_1,\ldots, \uu_{k-1})\sim P^{k-1}}\sum d(\uu_i)-\EE_{v} d(v) \right.\right.\\
&\mathrel{\phantom{= \EE_{S\sim P^{m-k+1}}\EE_{(\vv_1,\ldots, \vv_{k-1})\sim P^{k-1}}[\sup_{d\in \G_{\vv_{1:k-1}}}}}\left.\left.-\frac{1}{m} \EE_{(\uu_1,\ldots, \uu_{k-1})\sim P^{k-1}}\sum d(\uu_i)+ \frac{1}{m}\sum^{k-1}_{i=1}d(\vv_i)\right|\right]\\
\leq &\EE_{S\sim P^{m-k+1}}\EE_{(\vv_1,\ldots, \vv_{k-1})\sim P^{k-1}}\left[\sup_{d\in \G_{\vv_{1:k-1}}}\left|\EE_{(\uu_1,\ldots, \uu_{k-1})\sim P^{k-1}}\left[ \frac{1}{m}\sum_{v\in S\cup\{\uu_1,\ldots,\uu_{k-1}\}} d(v)\right]-\EE_{v}[d(v)] \right|\right.\\
&\mathrel{\phantom{= \EE_{S\sim P^{m-k+1}}\EE_{(\vv_1,\ldots, \vv_{k-1})\sim P^{k-1}}[\sup_{d\in \G_{\vv_{1:k-1}}}}}\left.+\left|\frac{1}{m} \EE_{(\uu_1,\ldots, \uu_{k-1})\sim P^{k-1}}\sum d(\uu_i)-\frac{1}{m} \sum^{k-1}_{i=1)}d(\vv_i)\right|\right]\\
\le & \EE_{(\uu_1,\ldots, \uu_{k-1})\sim P^{k-1}}\left[\EE_{S\sim P^{m-k+1}}\EE_{(\vv_1,\ldots, \vv_{k-1})\sim P^{k-1}}\left[\sup_{d\in \G_{\vv_{1:k-1}}}\left| \frac{1}{m}\sum_{v\in S\cup\{\uu_1,\ldots,\uu_{k-1}\}} d(v)-\EE_{v}[d(v)] \right|\right]\right]\\
&\mathrel{\phantom{= \EE_{S\sim P^{m-k+1}}\EE_{(\vv_1,\ldots, \vv_{k-1})\sim P^{k-1}}[\sup_{d\in \G_{\vv_{1:k-1}}}}}
+\frac{2k}{m}
\end{align*}

Renaming $\uu_1,\ldots, \uu_{k-1}$ and $\vv_1,\ldots, \vv_{k-1}$ we can write:
\begin{align*}
&\EE_{(\uu_1,\ldots, \uu_{k-1})\sim P^{k-1}}\left[\EE_{S\sim P^{m-k+1}}\EE_{(\vv_1,\ldots, \vv_{k-1})\sim P^{k-1}}\left[\sup_{d\in \G_{\vv_{1:k-1}}}\left| \frac{1}{m}\sum_{v\in S\cup\{\uu_1,\ldots,\uu_{k-1}\}} d(v)-\EE_{v}[d(v)] \right|\right]\right]+\frac{2k}{m}\\
=& \EE_{(\vv_1,\ldots, \vv_{k-1})\sim P^{k-1}}\left[\EE_{S\sim P^{m-k+1}}\EE_{(\uu_1,\ldots, \uu_{k-1})\sim P^{k-1}}\left[\sup_{d\in \G_{\uu_{1:k-1}}}\left| \frac{1}{m}\sum_{v\in S\cup (\vv_1,\ldots,\vv_{k-1})} d(v)-\EE_{v}[d(v)] \right|\right]\right]+\frac{2k}{m}\\
=&\EE_{(\uu_1,\ldots, \uu_{k-1})\sim P^{k-1}}\EE_{S\sim P^{m}} \left[\sup_{d\in \G_{\uu_{1:k-1}}}\left| \frac{1}{m}\sum_{v\in S} d(v)-\EE_{v}[d(v)] \right|\right]+\frac{2k}{m}
\end{align*}
Finally we apply. \cref{thm:expectedvapnik}. Recalling that
$\gvcdim(\D_{\uu_{1:k-1}})=\vcd$,  and that the sequence $S$ is
drawn IID independent of the choice ${\uu_{1:k-1}}$, we obtain for
every fixed $(\uu_1,\ldots, \uu_k)$
\[\EE_{S\sim P^m}\left[\sup_{d\in \D_{\uu_{1:k-1}}}\left| \frac{1}{m}\sum_{v\in S} d(v)-\EE_{v}[d(v)] \right|\right] \le \frac{4+\sqrt{\vcd\log 2 em/\vcd}}{\sqrt{2m}}\]
\paragraph{Bounding **}
\begin{align*}
 &\EE_{S\sim P^m } \left[\sup_{g\in \G} \left| \frac{1}{m^{k-1}} \sum_{\vv_{1:k-1}\in S^{k-1}} \EE_{v} g_{\vv_{1:k-1}}(v) - \EE_{\vv_{1:k-1}} \EE_{v} g_{\vv_{1:k-1}}(v)\right|\right]\\
 \le &
 \EE_{v}\EE_{S\sim P^m}\left[\sup_{g\in \G} \left| \frac{1}{m^{k-1}} \sum_{\vv_{1:k-1}\in S^{k-1}} g_{\vv_{1:k-1}}(v) - \EE_{\vv_{1:k-1}}  g_{\vv_{1:k-1}}(v)\right|\right]\\
 =&\EE_{v}\EE_{S\sim P^m}\left[\sup_{g\in \G} \left| \frac{1}{m^{k-1}} \sum_{\vv_{1:k-1}\in S^{k-1}} g_{v}(\vv_1,\ldots, \vv_{k-1}) - \EE_{\vv_{1:k-1}}  g_{v}(\vv_1\ldots, \vv_{k-1})\right|\right]\\
  =&\EE_{v}\EE_{S\sim P^m}\left[\sup_{g\in \G_v} \left| \frac{1}{m^{k-1}} \sum_{\vv_{1:k-1}\in S^{m}} g(\vv_1,\ldots, \vv_{k-1}) - \EE_{\vv_{1:k-1}}  g(\vv_1\ldots, \vv_{k-1})\right|\right] \end{align*}
  We now use the induction hypothesis: Note that $\G_{v}$ is $(k-1)$-distinguishing class with $\gvcdim(\G_{v})\le \vcd$ for every choice of $v$. Thus, fixing $v$:

\begin{align*}
&\EE_{S\sim P^m}\left[\sup_{g\in \G_v} \left| \frac{1}{m^{k-1}}
\sum_{\vv_{1:k-1}\in S^{k-1}} g(\vv_1,\ldots, \vv_{k-1}) -
\EE_{\vv_{1:k-1}}  g(\vv_1\ldots, \vv_{k-1})\right|\right]\\  \le
&\frac{(k-1)\left(4+ \sqrt{\vcd \log (2em/\vcd)}\right)}{\sqrt{2m}}+
\frac{k(k-1)}{m}
\end{align*}

\paragraph{Continuing the proof}
With the aforementioned bound on the terms * and ** we now obtain

\begin{align*}
* + ** &\le  \frac{4+\sqrt{\vcd\log 2em/\vcd}}{\sqrt{2m}}+\frac{2 k}{m}+ \frac{(k-1)\left(4+ \sqrt{\vcd \log (2em/\vcd)}\right)}{\sqrt{2m}}+ \frac{k(k-1)}{m}\\
&=\frac{ k\left(4+ \sqrt{\vcd \log (2em/\vcd)}\right)}{\sqrt{2m}}+
\frac{(k+1)k}{m}
\end{align*}

\subsection{Proof of \cref{thm:samplelower}}\label{prf:samplelower}
\low*
To prove \cref{thm:samplelower} we will in fact prove a stronger statement: We will show that it is not only hard to compute a $g\in \G$ as required, but in fact it is even hard to determine if such $g$ exists vs. the case that $p_1=p_2$.

Specifically let us call an algorithm $A$ a testing algorithm for $\G$ with sample complexity $m(\epsilon,\delta)$ if $A$ receives IID samples from two distributions $p_1$ and $p_2$ of size $m(\epsilon,\delta)$ and returns either $\yes$ or $\no$ such that w.p. $(1-\delta)$:
\begin{itemize}
\item If $p_1=p_2$ the algorithm returns $\yes$.
\item If $\ipm_{\G}(p_1,p_2)>\epsilon$ the algorithm returns $\no$
\end{itemize}

\begin{theorem}\label{thm:strongsamplelower}
Let $\G$ be a $k$--distinguishing class with $\gvcdim(\G)=\vcd$. Any testing algorithm $A$
with sample complexity $m(\epsilon,\delta)$ must observe $\Omega\left(\frac{\sqrt{\vcd}}{2^{7k^3}\epsilon^2}\right)$ examples for any $\delta< \frac{2^{-k\log k}}{3}$.
\end{theorem}

Clearly, \cref{thm:samplelower} is a corollary of \cref{thm:strongsamplelower}. Indeed if $A$ is a discriminating algorithm for $\G$ with sample complexity $m(\epsilon,\delta)$ we can apply it over a sample of size $m(\epsilon/3,\delta)$ to receive (w.p. $1-\delta$) a graph $g$ s.t.
\[\ipm_{\G} (p_1,p_2) \le |\L_{p_1}(g)-\L_{p_2}(g)|+\frac{\epsilon}{3}.\]
With an additional sample of size $O(\frac{k^2 \log 1/\delta}{\epsilon^2})$ we can estimate $|\L_{p_1}(g)-\L_{p_2}(g)|$ within accuracy $\epsilon/3$, and verify if $\ipm_{\G}(p_1,p_2)<\epsilon$: The test will then output $\yes$ if $|\L_{p_1}(g)-\L_{p_2}(g)|<\frac{\epsilon}{3}$. 

To conclude, we constructed a testing algorithm with sample complexity $m(\epsilon,\delta)+ C \frac{k^2 \log 1/\delta}{\epsilon^2}$. Assuming $\rho$ is sufficiently large, in particular $\frac{\sqrt{\rho}}{2^{7k^3}} \gg k^3\log k$, we obtain that $m(\epsilon,\delta)=\Omega(\frac{\sqrt{\rho}}{2^{7k^3 \epsilon^2}})$, if $\delta< \frac{2^{-k\log k}}{3}$.

We proceed with the proof of \cref{thm:strongsamplelower}.

\subsubsection{Proof of \cref{thm:strongsamplelower}}
 The proof is done by induction. For the induction, we will assume a more fine-grained lower bound. We will assume that there exists a constant $C$ so that for every $n\le k-1$, if $m_{n}(\epsilon,\delta)$ is the sample complexity of a testing algorithm for an $n$-distinguishing class then: \begin{align}\label{eq:mlowerbound}m_n(\epsilon,\delta)\ge  C \frac{\sqrt{\vcd}}{(n+1)!2^{\sum_{j=1}^{n} 6j^2}\cdot\epsilon^2}=\Omega\left(\frac{\sqrt{\vcd}}{2^{7n^3}\epsilon^2}\right).\end{align}
$C>0$ will depend only on the constant for the lower bound for testing if two distributions are distinct or $\epsilon$-far in total variation, as in \cref{thm:distest}.

We start with the case $k=1$.

\underline{$k=1$} The case $k=1$ follows directly from \cref{thm:distest}. Let $\D$ be a class with VC dimension $\vcd$. by restricting our attention to probabilities supported on the shattered set of size $\vcd$, we may assume that $|\V|=\vcd$ and that $\D=P(\V)$. Note then, that for the IPM distance we then have
\[ \ipm_{\D}(p_1,p_2)= \tv{p_1}{p_2}.\]
\cref{thm:distest} immediately yields the result.

\underline{the induction step}
We now proceed with the proof assuming the statement holds for $k-1$.

By assumption $\gvcdim(\G)=\vcd$. Fix $v\in \V$ such that $\gvcdim(\G_{v})=\vcd$. For every $q\in (0,1)$ and distribution $p$ denote
\begin{equation}\label{eq:qv}p^q:= q\delta_{v}+(1-q)p.\end{equation}

We next state the core Lemma we will need for the proof:
\begin{lemma}\label{lem:lowerbound}
Let $\G$ be a family of $k$-hypergraphs and $p_1,p_2$ two distributions. Assume that for some $v\in \V$ we have that:
\[ \ipm_{\G_{v}}(p_1,p_2) \ge \epsilon.\] Let $p^q_1$ and $p^q_2$ be as in \cref{eq:qv} for our choice of $v\in \V$.

Then for some value $q\in \{0,\frac{1}{k},\frac{2}{k},\cdots 1\}$ we have that,
 \[ \ipm_{\G}(p_1^q,p_2^q) \ge \frac{\epsilon}{2^{3k^2}}.\]
\end{lemma}
We deter the proof of \cref{lem:lowerbound} to \cref{prf:lowerbound}, and proceed with the proof of the induction step. Let us denote $\delta_k = 2^{-k\log k}$ and denote $c_k=2^{-3k^2}$.

Let $A$ be a testing algorithm for $\G$ with sample complexity $m(\epsilon,\delta)$ as in \cref{thm:strongsamplelower}.
We can now construct a testing algorithm for $\G_{v}$ with sample complexity \begin{equation}\label{eq:samplecomplexity} m_{k-1}(\epsilon,\delta)=(k+1)\cdot m(c_k\epsilon,\frac{\delta}{k}),\end{equation} as follows: Run the testing algorithm $A$ on pairs of distributions $(p_1,p_2), (p_1^{1/k},p_2^{1/k}),\ldots, (p_1^1,p_2^1)$, each on its own fixed sample of size $m(c_k\epsilon,\frac{\delta}{k})$. If the algorithm returns $\no$ for any of these tests, output $\no$, else output $\yes$.

We now show that if $p_1=p_2$ the algorithm outputs w.p. $(1-\delta)$ $\yes$: Indeed, since $p_1=p_2$, we have that $p_1^q=p_2^q$ for all $q$: Applying union bound we have that w.p. $(1-\delta)$ the algorithm indeed outputs $\yes$.

On the other hand, if $\ipm_{\G_{v}}(p_1,p_2)\ge \epsilon$ we have by \cref{lem:lowerbound} that for one of the distributions $(p_1^q,p_2^q)$, $\ipm_{\G}(p_1^q,p_2^q)>c_k \epsilon$ , in particular the algorithm will output $\no$ with probability $(1-\delta)$.
Overall we constructed a testing algorithm for $\G_{v}$ with sample complexity as in \cref{eq:samplecomplexity}.
Reparametrizing we obtain:
\[ m(\epsilon,\delta)= \frac{m_{k-1}(c^{-1}_k\epsilon,k\delta)}{k+1}.\]
If $k\delta <2^{-(k-1)\log (k-1)}$, in particular $\delta<2^{-k\log k}$: we obtain from the induction hypothesis that 

\begin{align*}m_{k-1}(c_{k}^{-1}\epsilon,k\delta)&\ge C \frac{\sqrt{\vcd}}{k!2^{\sum_{n=1}^{k-1} 6n^2}\cdot \left(2^{3k^2}\epsilon\right)^2}
\end{align*}
and the result immediately follows.

\subsubsection{Proof of \cref{lem:lowerbound}}\label{prf:lowerbound}
Denote
\[ \Delta^{g}_n(p_1,p_2) = \EE_{\uu_{1:n}\sim p^{k-n}_1} g(\underbrace{v,v,v,\ldots,v}_{ \mathrm{n~ times}}, \uu_1,\ldots, \uu_{k-n})- \EE_{\uu_{1:n}\sim p^{k-n}_2} g(\underbrace{v,v,v,\ldots,v}_{ \mathrm{n~ times}}, \uu_1,\ldots, \uu_{k-n})\]
One can show  that
\begin{align*}
\ipm_{\G}(p^q_1,p^q_2) &=\sup_{g\in \G}\left|\sum {k\choose n} q^n(1-q)^{k-n} \Delta^{g}_n (p_1,p_2)\right|\\
&= \sup_{g\in \G}\left|(1-q)^k \Delta^{g}_0(p_1,p_2)+kq(1-q)^{n-1}\Delta^{g}_1(p_1,p_2) + \sum_{n=2}^k {k\choose n} q^n(1-q)^{k-n} \Delta^{g}_n (p_1,p_2)\right|\\
&=\sup_{g\in \G}\left|\Delta_0^{g}(p_1,p_2) + kq\left(\Delta_1^g(p_1,p_2)-\Delta_0^g(p_1,p_2)\right) + q^2 p_g(q)\right|
\end{align*}

where $p_g(q)$ is some $k-2$ degree polynomial in $q$ whose coefficient depend on $g$ and $p_1$ and $p_2$. We next apply the following claim
\begin{claim}\label{cl:vandermonde}
Let $f(q) = a_0 + a_1q +  q^2 p(q)$ where $p(q)$ is some $k-2$ degree polynomial. then for some value $q_0 \in \{0,\frac{1}{k},\frac{2}{k},\cdots 1\}$ we have that
$|f(q_0)| \ge \frac{|a_1|}{2^{3k^2}}$
\end{claim}
\begin{proof}[Proof Sketch]
We provide a full proof for this claim in \cref{prf:vandermonde}. In a nutshell, \cref{cl:vandermonde} follows from the equivalence between norms in finite dimensional spaces. Indeed, the mapping
\[(a_0,\ldots, a_{k})\to (p_a(1/k),p_a(2/k),\ldots, p_a(1)),\]
where $p_a(x)=\sum a_i x^i$ is known to be a non--singular linear transformation induced by the appropriate Vandermonde matrix (specifically. $V_{i,j}= ((i-1)/k))^{j-1}$). Letting $\lambda_{min}$ be the smallest singular value of the matrix $V$, we know that $\|V\aa\|_{2}\ge \lambda_{min}\|\aa\|_{2}$. where $\aa$ is the vector of coefficients of the polynomial $p_{a}$.

Finally, we exploit the relation in $\mathbb{R}^{k+1}$: $\|x\|_{\infty} \le \|x\|_{2}\le \sqrt{k+1}\|x\|_{\infty}$. We can, thus, relate the max norm of the coefficient vector $\|a\|_{\infty}\ge |a_1|$ to the maximum value $\max_{i\in \{0,\ldots, k\}} \sum a_j (i/k)^{j}=\|V\aa\|_{\infty}$ to obtain
\[ |a_1|\le \|\aa\|_{2}\le \lambda^{-1}_{min}\|V\aa\|_2\le \frac{\sqrt{k+1}}{\lambda_{min}}\|V\aa\|_{\infty}=\frac{\sqrt{k+1}}{\lambda_{min}}\max_{i\in \{0,\ldots, k\}} \sum a_j (i/k)^{j}\]
 It remains only to lower bound the singular values of $V$,
this is done in the full proof in \cref{prf:vandermonde}.
\end{proof}
With \cref{cl:vandermonde} in mind we prove the result as follows:
First, suppose that for some $g\in \G$ we have that \[|k(\Delta_0^g(p_1,p_2)-\Delta_1^g(p_1,p_2)|>\frac{\epsilon}{2}.\] In this case, applying \cref{cl:vandermonde} with $a_0= \Delta_0^{g}(p_1,p_2)$ and $a_1 =k\left(\Delta_0^g(p_1,p_2)-\Delta_1^g(p_1,p_2)\right)$ and $p=p_g$, we obtain that
there exists a value $q=j/k$ such that
$\ipm_{\G}(p^q_1,p^q_2) \ge \frac{\epsilon}{2^{3k^2}}$.

On the other hand, consider the case that \[|k(\Delta_0^g(p_1,p_2)-\Delta_1^g(p_1,p_2)|<\frac{\epsilon}{2}.\] For any $g\in \G$, by assumption we have that $|\Delta_1^g(p_1,p_2)|>\epsilon$, for some $g\in \G$. Hence $|\Delta_0^g(p_1,p_2)|>\epsilon/2$. By definition of $\Delta_0$ we have that for $q=0$ we obtain that:
$\ipm_{G}(p^q_1,p^q_2)=|\L(p_1^{q})-\L(p_2^{q})|>\frac{\epsilon}{2}$.
\section{Expressivity -- Proofs}\label{prfs:expressivity}
\subsection{Proof of \cref{thm:main_expressive1}}\label{prf:main_expressive1}
\expressive*
As stated, the class $\G$ will consist of a single graph $g$.
The graph $g$ is going to be a bipartite graph. We thus, divide the vertices into two infinite sets: $\V_1$ and $\V_2$ the elements of $\V_1$ will be indexed by $\mathbb{N}$ i.e. $\V_1=\{v_1,v_2,\cdots\}$ and we index the elements of $\V_2$ with finite subsets of $\mathbb{N}$ $V_2=\{v_{A}: A\subseteq \mathbb{N}, |A|<\infty\}$. Next we define $g$ so that an edge passes between $v_i\in \V_1$ and $v_{A}\in \V_2$ iff $i\in A$.

Let $\D$ be a distinguishing class with finite sample complexity, in particular $\gvcdim(\D)<\infty$. Denote $\gvcdim(\D)=\vcd$. Let $\D_{1}$ be the restriction of $\D$ to $\V_1$: Note that $\gvcdim(\D_1)\le \vcd$.

Next we make the following claim:
\begin{claim}\label{cl:disjoint}
There are two distributions, $q_1$ and $q_2$, supported on $\V_1$ so that
\[\ipm_{\D_1}(p_1,p_2)<\epsilon.\]
and yet $q_1$ and $q_2$ have disjoint support.
\end{claim}
\begin{proof}
To construct two such distributions, choose a set $S\subseteq \V_1$ of size $m$ large enough (to be determined later). Then, randomly choose two samples $S_1$ and $S_2$ out of $S$ (uniformly), each of size $O(\frac{\vcd}{\epsilon^2})$. Then, by \cref{thm:uc} with some constant probability we have that $\ipm_{\D}(p_{S_1},p_S)<\epsilon/2$ and similarly $\ipm_{\D}(p_{S},p_{S_2})<\epsilon/2$ . Taken together we obtain that $\ipm_{\G}(p_{S_1},p_{S_2})<\epsilon$.

Also, if $S$ is sufficiently large (say, of order $O(\frac{\vcd^2}{\epsilon^4})$),
 we would have that w.h.p $S_1\cap S_2=\emptyset$. Thus, let $q_1=p_{S_1}$ and $q_2=p_{S_2}$.
\end{proof}
With \cref{cl:disjoint}, we proceed with the proof. Let $q_1$ and $q_2$ be as in \cref{cl:disjoint}. Let $A$ be the support of $q_1$, and define $p_1$ to be a distribution $p_1= \frac{1}{2}\delta_{A} + \frac{1}{2} q_1$ and similarly we define $p_2=\frac{1}{2}\delta_{A}+\frac{1}{2}q_2$. We then have \begin{align*}\ipm_{\D}(p_1,p_2)&=
\frac{1}{2}\ipm_{\D}(q_1,q_2)\\
&=\frac{1}{2}\ipm_{\D_1}(q_1,q_2)\\
&<\epsilon.\end{align*}

On the other hand, note that for $p_1$ the probability to draw an edge from $g$ is at least $1/2$ (indeed if $v_1=v_A$ and $v_2\ne v_A$ drawn from $q_1$ then $g(v_1,v_2)=1$. On the other hand, the probability to draw an edge from $p_2$ is $0$. It follows that
\[\ipm_{\G}(p_1,p_2)>\frac{1}{2}.\]

\subsection{Proof of \cref{thm:main_expressivek}}\label{prf:main_expressivek}
\expressivek*
The construction is similar to the case $k=2$. We again divide the vertices into two infinite sets $\V_1$ and $\V_2$. Again, the elements of $\V_1$ will be indexed by $\mathbb{N}$, and the elements of $\V_2$ are indexed by finite subsets of $\mathbb{N}$. $\V_2=\{v_{A}: A\subseteq \mathbb{N}, |A|<\infty\}$.

We define the hyper graph $g_{k}$ to be a (undirected) graph that contains a hyperedge $(v_{i_1},\ldots, v_{i_{k-1}},v_A)$ whenever $\{i_1\ldots, i_{k-1}\}\subseteq A$.

Next, as before we construct two distributions with distinct support such that $\ipm_{\G}(p_1,p_2)\le \epsilon$. This is done similar to the proof of \cref{thm:main_expressive1}. Specifically:
\begin{claim}\label{cl:disjointk}
Let $\G$ be a $k-1$-distinguishing class defined on $\V_1$. There are two distributions, $q_1$ and $q_2$, supported on $\V_1$ so that
\[\ipm_{\G}(p_1,p_2)<\epsilon.\]
and yet $q_1$ and $q_2$ have disjoint support.
\end{claim}
The proof is a repetition of the proof of \cref{cl:disjoint}, where we draw $S_1$ and $S_2$ to be order of $O(\frac{k^2\vcd}{\epsilon^2})$, and again invoke \cref{thm:uc}.

As before, then, given a class $\G$ of $k-1$--hypergraphs we take two distributions $q_1$ and $q_2$ as in \cref{cl:disjointk} and if $A$ is the support of $q_1$, we take $p_1=\frac{1}{k}\delta_{v_{A}}+(1-\frac{1}{k})q_1$ and let $p_2=\frac{1}{k}\delta_{v_{A}}+(1-\frac{1}{k})q_2$. Then, we can show that $\ipm_{\G}(p_1,p_2)\le \epsilon$. On the other hand, the probability to draw an edge from $g_k$ is $k\cdot \frac{1}{k}(1-\frac{1}{k})^{k-1}\ge e^{-1}$ according to $p_1$, but the probability to draw an edge from $p_2$ is $0$.

\subsection{Proof of \cref{thm:main_expressive_finite_k}}\label{prf:main_expressive_finite_k}
\expressivefinitek*
The proof is similar to the proof of \cref{thm:main_expressivek}. For simplicity, let us assume that $|\V|= n+\log n$. This will not change the results up to constants.

Given $n+ \log n$ vertices we partition them into two sets $\V_1$, of size $\log n$ and $\V_2$. We index the elements of $\V_1$ as $\{v_1,\ldots, v_{\log n}\}$ and we index the elements of $\V_{2}$ with subsets of $[\log n]$. We then consider a graph $g$ that contains only hyper-edges of the form $(v_{i_1},\ldots, v_{_{k-1}},v_{A})$ iff $\{i_1,\ldots, i_{k-1}\}\in A$.

Next, let $\G_{k-1}$ be a distinguishing class with $\gvcdim(\G_{k-1})=\vcd$, and let $m(\epsilon,\delta)=O\left(\frac{\vcd k^2}{\epsilon^2}\right)$ be an upper bound on the sample complexity of classes of graph VC dimension $\vcd$.

We claim that if $\log n \ge m^2(\epsilon/8,0.99)$ then there are two distinct distributions $q_1,q_2$ over $[\log n]$, with disjoint support such that  $\ipm_{\G_{k}}(q_1,q_2)<\epsilon$. The proof is done as in \cref{cl:disjointk}.

Indeed, we draw IID, and uniformly, two random samples $S_1$ and $S_2$ from $\{1,\ldots, \log n\}$ of size $m(\epsilon/8,0.99)$.
One can show that w.p $1/4$ we have that $S_1\cap S_2$ are distinct, also we have w.p $0.98$ that $\ipm_{\G}(p_S,p_{S_1})<\epsilon/8$ and similarly $\ipm_{\G}(p_S,p_{S_2})<\epsilon/8$. Taken together we obtain that with positive probability $q_1=p_{S_1}$ and $q_2=p_{S_2}$ have disjoint support and $\ipm_{\G}(q_1,q_2)<\frac{\epsilon}{4}$.

As in \cref{thm:main_expressivek}, let $A$ be the support of $q_1$ and consider a distribution
$p_1=\frac{1}{k} \delta_{v_{A}}+(1-\frac{1}{k}) q_1$ and similarly $p_2=\frac{1}{k} \delta_{v_{A}}+(1-\frac{1}{k} q_2$. One can show that $\ipm_{\G}(p_1,p_2)<\frac{\epsilon}{4}$ but the probability to draw an edge from $g$ according to $q_1$ is at least $1/4$, while it equals $0$ if we draw edges according to $p_2$.

To conclude, we showed that if  $\log n \ge m^2(\epsilon/8,0.99)$ then $\ipm_{\G_k}\prec \epsilon \ipm_{\G}$. In other words, if $\ipm_{\G_k}\succ \epsilon \cdot \ipm_{\G}$ then
$\log n \le m^2(\epsilon/8,0.99)$.
\[ \vcd= \Omega\left(\frac{\epsilon^2}{k^2} \sqrt{\log n}\right).\]
\subsection{Proof of \cref{thm:main_expressive_finite_1}}\label{prf:main_expressive_finite_1}
\expressivefinite*
The proof is similar to the proof of \cref{thm:main_expressive1} but we will use an improved upper bound on the size of $S$ which we next state (see \cref{prf:disjoint_improved} for a proof):
\begin{lemma}\label{cor:disjoint_improved}
Let $\D$ be a class with $\gvcdim(\D)=\vcd$ over a domain $S$. There exists a constant $c>0$ (independent of $\D$ and $d$) such that if $|S| >c\cdot \frac{d}{\epsilon^2}\log^2 (d/\epsilon^2)$, Then there are two distributions $q_1$ and $q_2$, supported on $S$ such that:
\begin{enumerate}
\item $q_1$, and $q_2$ have disjoint support.
\item $\ipm_{\D}(q_1,q_2)<\epsilon$
\end{enumerate}
\end{lemma}

The graph $g$ is constructed as in \cref{thm:main_expressive1}. Let $\V$ be a set of vertices of size $n+\log n$, let $\V_1$ be a set of size $\log n$ and we index its elements with $\{v_1,\ldots, v_2,\ldots, v_{\log n}\}$. We let $\V_2$ include all other elements and we index them via subsets of $[\log n]$. The graph is again constructed so that $v_{A}\in \V_2$ has an edge to $v_i\in \V_1$ iff $i \in A$. As before, we make the graph bipartite, i.e. both $\V_1$ and $\V_2$ are independent sets.

Now suppose $\log n \ge c\frac{\vcd}{\epsilon^2}\log^2 \frac{d}{\epsilon^2}$.
By \cref{cor:disjoint_improved} we have that there exists a set $A\subseteq \{1,\ldots, \log n\}$, a distribution $p_1$ and $p_2$ where $p_1$ is supported on $A$ and $p_2$ is supported on its compelement so that $\ipm_{\G}(p_1,p_2)<\epsilon$. As before we construct $q_1= \delta v_{A} + (1-\delta)p_1$ and $q_2=\delta v_{A}+ (1-\delta)p_2$. One can verify that $\ipm_{\G}(q_1,q_2)<\epsilon$ but $\ipm_{\G_{k+1}}(q_1,q_2)>\frac{1}{2}$. Thus, if $\ipm_{\G_k}\succ \epsilon \cdot \ipm_{\G_{k+1}}$ then $\log n \le c \frac{\vcd}{\epsilon^2}\log^2 \frac{d}{\epsilon^2}$. In turn $d=\tilde \Omega (\epsilon^2 \log n)$.

\subsection{Proof of \cref{cor:disjoint_improved}}\label{prf:disjoint_improved}

First w.l.o.g we assume that the constant functions are in $\D$ (i.e. $0$ and $1$).

We want to choose a constant $c$ so that if $|S|\ge c\frac{2\vcd}{\epsilon^2}\log^2 \frac{2\vcd}{\epsilon^2}$, then we have
$\frac{|S|}{\ln^2|S|} > \frac{2\vcd\log e}{\epsilon^2}$. Fix such $c>0$, and let $\mathcal{H}_{m} = \{\mathrm{sign}(\sum_{i=1}^m (2d_i(v)-1)):~ d_{i}\in \D\}$ and denote $\mathcal{H}= \mathcal{H}_{\frac{2}{\epsilon^2} \ln |S|}$. Note that
\begin{align*}|\mathcal{H}| &\le |\D|^{\frac{2}{\epsilon^2}\ln |S|}\\
& \le |S|^{\frac{2\vcd}{\epsilon^2}\ln |S|} & \textrm{Sauer's Lemma}\\
& = 2^{\frac{2\vcd\log e}{\epsilon^2} \ln^2 |S|} \\
& < 2^{|S|}
\end{align*}
It thus follows that there exists $f\notin \mathcal{H}$. Let $f$ be such and define a matrix $M=\{0,1\}^{|S|\times |\D|}$ so that

\[M_{v,d}=\begin{cases} 1 & d(v)\ne f(v) \\
0 & \mathrm{else}
\end{cases}\]
Now suppose that for some distribution $q$ over $S$, for every $d$ we have that $\EE_{v\sim q}[d(v)=f(v)]<\frac{1}{2}+\frac{1}{\epsilon}$. Then, defining $q_1=q(\cdot |f(v)=0)$ and $q_2=q(\cdot|f(v)=1)$ yields the desired result.
Indeed,
\begin{align*} \sup_{d\in \D}| \L_{q_1}[d]-\L_{q_2}[d]| &=\sup_{d\in \D}2|\frac{1}{2}\L_{q_1}[d]-\frac{1}{2}\L_{q_2}[d] |\\
&\ge \sup_{d\in \D}2|q(f(v)=1) \L_{q_1}[d]-q(f(v)=-1)\L_{q_2}[d] |-4\max_{y\in \{1,-1\}}\{|\frac{1}{2}-q(f(v)=y)|\\
 &\ge \sup_{d\in \D}2|\EE_{(v,y)\sim q} yd(v) |-4\epsilon\\
 &= \sup_{d\in \D}2|1-2q(d(v)\ne f(v)) |-4\epsilon\\
 & \ge 8\epsilon.
\end{align*}

We now wish to prove that indeed, such a $q$ exists.
Suppose, otherwise: That for any distribution $q$ over $S$ we can find $d$ such that $\EE_{v\sim q}[d(v)=f(v)]>\frac{1}{2}+\frac{1}{\epsilon}$. This can be rephrased in terms of a value of a minimax game as follows:

\[ \max_{q\in \Delta(S)} \min_{d\in \D} q^\top M_{d} < \frac{1}{2}-\epsilon,\]
Where $\Delta(S)$ denotes the set of distributions over $S$.
It is well known (\cite{lipton1994simple}, thm 2), that for any game defined by any matrix $M$ with $c$ columns, there exists a strategy for the row player that chooses uniformly from a multiset of $\frac{\ln c}{2\epsilon^2}$ and achieves $\epsilon$-optimiality.

In our setting, this translate to a uniform distribution $p$, supported on $\frac{\ln |S|}{2\epsilon^2}$ distinguishers $\{d_1,\ldots, d_\frac{\ln |S|}{2\epsilon^2}\}$
such that \[\frac{2\epsilon^2}{\ln |S|}\sum_{d_i}[d_i(v)\ne f(v)] < \frac{1}{2},\] this contradicts the fact that $f\notin  \mathcal{H}$.

We thus obtain that there exists a distribution $q$ over $S$ so that for every $d\in \D$ $\sum q(v)[d(v)\ne f(v)]>\frac{1}{2}-\epsilon$.
\section{Additional Proofs}

\subsection{Proof of \cref{cl:vandermonde}}\label{prf:vandermonde}
Consider the Vandermonde Matrix $V\in M_{k+1,k+1}$ given by $V_{i,j}=\left(\frac{i-1}{k}\right)^{j-1}$. Our first step will be to lower bound the smallest singular value of $V$. In turn, we will obtain a lower bound on the maximum value over the coordinates of the vector $V\aa$. The proof can then be derived from the identity: $(V\aa)_{i}= \sum_{j=1}^{k+1} a_j \left(\frac{i-1}{k}\right)^{j} $.

Let $\lambda_1\le \lambda_2\le\ldots\le \lambda_{k+1}$ be the singular values of $V$.
To bound the smallest singular value, $\lambda_1$, we first observe that $\lambda_{k+1}$-- the highest singular value is bounded by $k+1$. To see that $\lambda_{k+1}\le k+1$, observe that for any vector $\|\aa\|\le 1$ we have that \[\|V\aa\|_2\le k+1\max |V_{i,j}||a_i| \le k+1.\]

Next, using the formula for the determinant of a Vandermonde matrix, and the relation $\det(V)=\prod \lambda_i$, we obtain:
\begin{align*}
\prod_{i=1}^{k+1} |\lambda_i |&= |\det(V)|\\
&= \prod_{1\le i< j\le k+1} \frac{|i-j|}{k}\\
&\ge 2^{-\frac{k(k-1)\log k}{2}}
\end{align*}
Taken together we obtain
\begin{align*}
\lambda_{min}&\ge \frac{2^{-\frac{k(k-1)\log k}{2}}}{\prod_{i=2}^{k+1}\lambda_i}\\
&\ge  \frac{2^{-\frac{k(k-1)\log k}{2}}}{\lambda_{k+1}^k}\\
& \ge 2^{-k(k-1)\log k-k \log (k+1)}\\
&= 2^{-k^2 +k \log k/k+1}\\
& \ge 2^{-2k^2}
\end{align*}

Finally, for any polynomial $p=\sum a_i q^i$ with coefficient $|a_1|$ we have that $\|\aa\|_2\ge |a_1|$. We thus obtain, \begin{align*}\max_{i} |p(\frac{i}{q})| &\ge \frac{1}{\sqrt{k+1}} \sqrt{\sum |p(\frac{i}{q}|)^2}\\
& = \frac{1}{\sqrt{k+1}}\|V\aa\|_2\\
& \ge \frac{1}{\sqrt{k+1}} \lambda_{1} \|\aa\|_2\\
& \ge 2^{-2k^2-1/2\log(k+1)} |a|_1 \\
& \ge 2^{-3k^2} |a|_1
\end{align*}